\documentclass{article}
\usepackage{ijcai25}

\usepackage{times}
\usepackage{soul}
\usepackage{url}
\usepackage[hidelinks]{hyperref}
\usepackage[utf8]{inputenc}
\usepackage[small]{caption}
\usepackage{graphicx}
\usepackage{amsmath}
\usepackage{amsthm}
\usepackage{booktabs}
\usepackage[switch]{lineno}

\usepackage{amsxtra, amsfonts, amssymb, amstext, mathtools}
\usepackage{thmtools}
\usepackage{thm-restate}
\usepackage{nicefrac}
\usepackage{xspace}
\usepackage[usenames,dvipsnames]{xcolor}
\usepackage[algo2e,ruled,vlined,linesnumbered]{algorithm2e}
\usepackage{wrapfig}
\usepackage{enumitem}
\usepackage[noabbrev, nameinlink]{cleveref}

\crefname{assumption}{assumption}{assumptions}
\crefname{observation}{observation}{observations}

\allowdisplaybreaks[4]
\newtheorem{theorem}{Theorem}

\newtheorem{corollary}[theorem]{Corollary}

\newcommand{\cocz}{\textsc{COCZ}\xspace}

\newcommand{\ojzjk}{\textsc{OJZJ}_k\xspace}
\newcommand{\ojzj}{\textsc{OJZJ}\xspace}
\newcommand{\omm}{\textsc{OMM}\xspace}

\newcommand{\ones}[1]{\left|#1\right|_1\xspace}
\newcommand{\onesa}[1]{g_1(#1)\xspace}
\newcommand{\onesb}[1]{g_2(#1)\xspace}

\newcommand{\R}{\ensuremath{\mathbb{R}}}

\newcommand{\N}{\ensuremath{\mathbb{N}}} 
\newcommand{\Z}{\ensuremath{\mathbb{Z}}}
\newcommand{\E}{\ensuremath{\mathbb{E}}}

\let\originalleft\left
\let\originalright\right
\renewcommand{\left}{\mathopen{}\mathclose\bgroup\originalleft}
\renewcommand{\right}{\aftergroup\egroup\originalright}

\DeclarePairedDelimiter{\nbOnes}{|}{|_1}
\DeclarePairedDelimiter{\floor}{\lfloor}{\rfloor}
\DeclarePairedDelimiter{\ceil}{\lceil}{\rceil}

\title{Tight Runtime Guarantees From Understanding the Population Dynamics of the GSEMO Multi-Objective Evolutionary Algorithm}

\author{
Benjamin Doerr$^1$
\and
Martin~S. Krejca$^1$\And
Andre Opris$^2$\\
\affiliations
$^1$Laboratoire d'Informatique (LIX), CNRS, École Polytechnique, Institut Polytechnique de Paris\\
$^2$University of Passau\\
\emails
\{firstname.lastname\}@polytechnique.edu,
andre.opris@uni-passau.de
}
\begin{document}
{\sloppy

\maketitle

\begin{abstract}
  The global simple evolutionary multi-objective optimizer (GSEMO) is a simple, yet often effective multi-objective evolutionary algorithm (MOEA). By only maintaining non-dominated solutions, it has a variable population size that automatically adjusts to the needs of the optimization process.  The downside of the dynamic population size is that the population dynamics of this algorithm are harder to understand, resulting, e.g., in the fact that only sporadic tight runtime analyses exist.
  In this work, we significantly enhance our understanding of the dynamics of the GSEMO, in particular, for the classic CountingOnesCountingZeros (COCZ) benchmark. From this, we prove a lower bound of order $\Omega(n^2 \log n)$, for the first time matching the seminal upper bounds known for over twenty years. We also show that the GSEMO finds any constant fraction of the Pareto front in time $O(n^2)$, improving over the previous estimate of $O(n^2 \log n)$ for the time to find the first Pareto optimum. Our methods extend to other classic benchmarks and yield, e.g., the first $\Omega(n^{k+1})$ lower bound for the OJZJ benchmark in the case that the gap parameter is $k \in \{2,3\}$. 
  We are therefore optimistic that our new methods will be useful in future mathematical analyses of MOEAs.
\end{abstract}


\section{Introduction}
\label{sec:introduction}

Many real-world optimization problems are characterized by several, often conflicting objectives. A common solution concept for such \emph{multi-objective optimization problems} is to compute a diverse set of \emph{Pareto optima} (solutions which cannot be improved in one objective without compromising in another objective) and let a human decision maker select one of these. Due to their population-based nature, \emph{multi-objective evolutionary algorithms (MOEAs)} are among the most prominent approaches to such problems and have found applications in numerous subfields of multi-objective optimization~\cite{CoelloLV07,ZhouQLZSZ11}.

The mathematical runtime analysis of MOEAs was started around 20 years ago~\cite{LaumannsTZWD02,Giel03,Thierens03}. It has gained considerable momentum in the last years, among others, with analyses of classic algorithms such as the NSGA-II, NSGA-III, SMS-EMOA, and SPEA2~\cite{ZhengLD22,WiethegerD23,BianZLQ23,RenBLQ24,Opris2025,OprisDNS24} or works discussing how MOEAs can solve submodular problems~\cite{QianYTYZ19,QianBF20,Crawford21,DoNNS23}.

The by far dominant algorithm in the rigorous analysis of MOEAs is the \emph{global simple evolutionary multi-objective optimizer} \cite{Giel03} (GSEMO). Due to its apparent simplicity, it was the first MOEA for which mathematical runtime analyses were conducted, and it is still often the first algorithm for which new phenomena are discovered, see, e.g., \cite{DinotDHW23,DangOS24} for recent examples. At the same time, it is a central algorithm, and many other algorithms, in particular in the area of submodular optimization, build on it. For example, algorithms such as POSS (Pareto Optimization for Subset Selection) \cite{QianYZ15nips}, POMC (Pareto Optimization for Monotone Constraints) \cite{QianSYTZ17monotone}, and POMS (Pareto Optimization for Multiset Selection) \cite{QianZTY18} are all variants of the GSEMO applied to a suitable bi-objective formulation of the submodular problem of interest.

Despite this impressive body of theoretical works on the GSEMO, a real understanding of the working principles of this algorithm is still missing. This is most visible from the fact that there are almost no lower bounds matching the existing runtime guarantees (see \Cref{sec:previousWorks} for a detailed discussion of this point). The reason is that for matching bounds, a deeper understanding of the population dynamics is necessary. This is particularly crucial for the GSEMO with its dynamic population size (note that the probability to choose a particular individual as parent is the reciprocal of the population size).

\paragraph{Our contribution:} In this work, we greatly expand our understanding of the population dynamics of the GSEMO. To this end, we study how this algorithm optimizes the classic CountingOnesCountingZeros (COCZ) benchmark. For readers less familiar with the area of runtime analyses, we note that it is the established approach of this field to study how a specific randomized search heuristics solves a well-understood benchmark problem, and from this derive insights into the working principles of the heuristic. All runtime analysis works cited above follow this approach.

We give more details on our new understanding of the population dynamics later (\Cref{sec:theory-results}) when we have made precise the GSEMO algorithm and the COCZ benchmark, and now only describe two implications. First, we indeed succeed in proving a lower bound of $\Omega(n^2 \log n)$ (\Cref{thm:gsemo-cocz-lower-bound}), which matches the classic upper bound of \cite{LaumannsTZWD02,Giel03}. This is the first tight lower bound for a benchmark problem in which reaching the Pareto front is non-trivial (as opposed to, e.g., the OneMinMax benchmark, where all  solutions are on the Pareto front).
Second, we also gain a deeper understanding on how difficult it is to reach the Pareto front. Whereas previously the time to find the first solution on the Pareto front was estimated by $O(n^2 \log n)$, we prove that $O(n^2)$ iterations suffice with high probability to reach the Pareto front and compute any linear fraction of it (\Cref{cor:time-to-spread}).

Our results are made possible by a number of new arguments. The most interesting one is that we add dummy individuals to the population to reach a population size equal to the maximum possible size. If such a dummy individual is chosen as parent, this iteration has no effect (but is counted as iteration). This argument helps overcome the changing population size of the original GSEMO. What is interesting is that this argument, which slows down the original algorithm, can be used to prove lower bounds on the runtime. The reason is that we exploit this argument not to estimate times directly (which is not possibly due to the unclear deceleration from the dummy individuals) but only to understand the shape of the population in the objective space. We are optimistic that this argument, and our other new proof ideas, will be useful in future runtime analyses of MOEAs as well. As a first support for this claim, we show that our methods also give a tight lower bound for the runtime of the GSEMO on the $\ojzj_k$ benchmark for all $k$ (where previous works could not analyze the cases $k=2$ and $k=3$) and that all our results extend to the SEMO algorithm.

\section{Previous Works}
\label{sec:previousWorks}

In the interest of space, we concentrate on the previous works most relevant for ours. For a general introduction to MOEAs and their success in applications, we refer to~\cite{CoelloLV07,ZhouQLZSZ11}.

We refer to~\cite{NeumannW10,AugerD11,Jansen13,ZhouYQ19,DoerrN20} for introductions to mathematical runtime analyses of randomized search heuristics. We note here that the typical approach in this area is to analyze, with mathematical means, how a specific heuristic solves a particular, often artificial, problem, and to derive from this analysis a deeper understanding of the working principles of the algorithm. Such works have successfully detected strengths or weaknesses of algorithms (e.g., the NSGA-II has intrinsic difficulties with three or more objective~\cite{ZhengD24many}), have proposed suitable settings for parameters (e.g., the cutoff time of automated algorithm configurators~\cite{HallOS22}), or have led to the design of novel algorithms (e.g., a variant of the SMS-EMOA with stochastic selection of the next parent population~\cite{BianZLQ23}).

Already the first mathematical runtime analysis of a MOEA proved an $O(n^2 \log n)$ runtime guarantee for the \emph{simple evolutionary multi-objective optimizer (SEMO)}, a predecessor of the GSEMO, on the \cocz benchmark \cite{LaumannsTZWD02}, see \cite{LaumannsTZ04} for the journal version. The same bound for the GSEMO followed a year later \cite{Giel03}. As we will see in this work, both bounds are tight. When looking at the proofs, both results estimate both the time to find the first Pareto optimum and the subsequent time to compute the full Pareto front by $O(n^2 \log n)$, whereas we shall show that the Pareto front is reached, and in fact any constant fraction of it is computed, in time $O(n^2)$. Since then, many more upper bounds on runtimes of the (G)SEMO were shown, and later also for more complex algorithms like the NSGA-II, only very few lower bounds exist, and these only apply to very specific situations.

The first lower bound, matching their own upper bound, is that the SEMO optimizes the LOTZ benchmark in time $\Omega(n^3)$ \cite{LaumannsTZWD02}. While clearly non-trivial, this result heavily exploits that the SEMO with its one-bit mutation operator cannot generate incomparable solutions until a solution on the Pareto front is found, and from that point on, the population always forms a contiguous interval of the Pareto front. With these restricted population dynamics, proving lower bounds was possible already in the first runtime analysis paper on MOEAs. For the GSEMO, the population dynamics are more complex. In particular, at any time, solutions not comparable with the parent can be generated. Consequently, despite attempts in \cite{DoerrKV13}, no interesting lower bounds exist for the GSEMO on LOTZ.

The first tight lower bound for the GSEMO on a classic benchmark was given by \cite{DoerrZ21aaai}, who showed that the GSEMO optimizes the $\ojzj_k$ benchmark in time $\frac 32 e n^{k+1} \pm o(n^{k+1})$ when the gap parameter satisfies $4 \le k = o(n)$. That such a tight bound is possible builds on particular properties of this benchmark. The Pareto front of the $\ojzj_k$ benchmark consists of an easy-to-explore inner part, from which two solutions are separated by difficult-to-cross gap of size $k$. When assuming $k\ge 4$ as in this result, it is easy to argue that the inner part is computed before the gaps are traversed, and hence the traversal of the gaps is slowed down by the then linear-size population. This argument breaks down for smaller values of $k$, and this is why no tight lower bounds existed in this case prior to this current work. The only other tight lower bound for the GSEMO on a classic benchmark we are aware of is the $\Omega(n^2 \log n)$ bound for the OneMinMax benchmark \cite{BossekS24}. This benchmark has the particularity that all solutions are Pareto optimal, hence the optimization process lacks the phase of advancing towards the Pareto front, which was the most demanding one in our work. Recently, lower bounds where proven for the runtime of the NSGA-II \cite{DoerrQ23LB}, but again, these only regard the OneMinMax and $\ojzj_k$ benchmarks; also, clearly, the population dynamics of the NSGA-II with its fixed population size are very different from the GSEMO. In summary, it is safe to say that there are very few interesting lower bounds for the GSEMO, and that this is caused by the difficulty of understanding the population dynamics of this algorithm.

\section{Preliminaries}
\label{sec:preliminaries}

We now provide some general notation and definitions for multi-objective optimization,  define the algorithms (\Cref{sec:preliminaries:gsemo}) and benchmark functions (\Cref{sec:preliminaries:benchmarks}) that used in this study, and state the mathematical tools needed in our analysis (\Cref{sec:preliminaries:math-tools}).

Let $\Z$ denote the integers, $\N \coloneqq \Z_{\ge 0}$ the natural numbers (including~$0$), and~$\R$ the reals.
For all $a, b \in \R$, let $[a .. b] \coloneqq [a, b] \cap \Z$ and $[a] \coloneqq [1 .. a]$.


We study pseudo-Boolean bi-objective maximization, that is, the maximization of \emph{objective functions} $f\colon \{0, 1\}^n \to \R^2$ of \emph{problem size} $n \in \N_{\geq 2}$.
We call each $x \in \{0, 1\}^n$ an \emph{individual}, and~$f(x)$ the \emph{objective value of~$x$}.
For all $i \in [n]$, we denote the value of~$x$ at position~$i$ by~$x_i$.

We compare objective values via the weak and strong \emph{dominance} relationships, which are (strict) partial orders.
For all objective values $u, v \in \R^2$, we say that~$u$ \emph{weakly dominates}~$v$ (written as $u \succeq v$) if and only if $u_1 \geq v_1$ and $u_2 \geq v_2$. If in addition $u \neq v$, then we say that~$u$ \emph{strictly dominates}~$v$ (written as $u \succ v$).
We say that~$u$ and~$v$ are \emph{incomparable} if neither weakly dominates the other.
We extend this terminology to individuals, where it then refers to the individuals' objective values.

Given an objective function~$f$, we call the set of maximal objective values (with respect to dominance) the \emph{Pareto front} of~$f$, that is, the set $\{f(x) \mid x \in \{0, 1\}^n \land \nexists y \in \{0, 1\}^n\colon f(y) \succ f(x)\}$.

\subsection{The Algorithms SEMO and GSEMO}
\label{sec:preliminaries:gsemo}

We study both the \emph{simple evolutionary multi-objective optimizer}~\cite{LaumannsTZWD02} (SEMO) and the \emph{global SEMO}~\cite{Giel03} (GSEMO), which only differ in how they create new solutions (\Cref{algo:GSEMO}).

The (G)SEMO maintains a \emph{population} of individuals, which will contain a maximum subset of non-dominated solutions among all solutions seen so far. This population is initialized with a single individual drawn uniformly at random from the search space.
In each iteration, one individual is selected uniformly at random (the \emph{parent}) and used to create a new individual (the \emph{offspring}) via \emph{mutation}, that is, a small random perturbation of the parent.
Afterward, the algorithm removes all individuals from the population that are weakly dominated by the offspring, and the offspring is added to the population if it is not strictly dominated by a member of the population.
This main loop is repeated until a user-defined termination criterion is satisfied.

\begin{algorithm2e}[t]
  \caption{
    The (G)SEMO algorithm~\protect\cite{LaumannsTZWD02,Giel03} for maximization of a given bi-objective function $f\colon \{0, 1\}^n \to \R^2$.
    The SEMO uses 1-bit mutation, the GSEMO standard bit mutation (see also \Cref{sec:preliminaries:gsemo}).
  }
  \label{algo:GSEMO}
  $x^{(0)} \gets$ an individual from $\{0, 1\}^n$ chosen uniformly at random\;
  $P^{(0)}=\{x^{(0)}\}$\;
  $t \gets 0$\;
  \While{\emph{termination criterion not met}}{%
    choose $x^{(t)}$ from $P^{(t)}$ uniformly at random\;\label{line:selectParent}
    $y^{(t)} \gets \mathrm{mutation}(x^{(t)})$\;\label{line:mutation}
    $Q^{(t)} \gets P^{(t)} \setminus \{z \in P^{(t)}\colon f(y^{(t)}) \succeq f(z)\}$\;
    \lIf{$\not\exists z \in Q^{(t)}\colon f(z) \succ f(y^{(t)})$}{%
      $P^{(t + 1)} \gets Q^{(t)} \cup \{y^{(t)}\}$%
    }
    \lElse{%
      $P^{(t + 1)} \gets Q^{(t)}$%
    }
    $t \gets t + 1$\;
  }
\end{algorithm2e}

The difference between the SEMO and the GSEMO is how they create the offspring~$y$ from the parent $x \in \{0, 1\}^n$.
The SEMO uses \emph{1-bit mutation}, which chooses a position $i \in [n]$ uniformly at random and copies~$x$ except for position~$i$, which is flipped to the other value.
That is, for all $j \in [n] \smallsetminus \{i\}$, we have $y_j = x_j$, for the $i$-th position we have $y_i = 1 - x_i$.
The GSEMO uses \emph{standard bit mutation}, which decides independently for each position whether to flip the bit (with probability~$\frac{1}{n}$) or not.
That is, for all $i \in [n]$ independently, we have $\Pr[y_i = x_i] = 1 - \frac{1}{n}$ and $\Pr[y_i = 1 - x_i] = \frac{1}{n}$.

\textbf{Runtime.}
As common in the runtime analysis of MOEAs, we define the \emph{runtime} of the (G)SEMO maximizing~$f$ as the (random) number of evaluations of~$f$ until the objective values of the population contain the Pareto front of~$f$ for the first time (we say that the population \emph{covers} the Pareto front).
To this end, we assume that the objective value of an individual is evaluated once, namely when it is created. For our definition of runtime to make sense, we assume that the algorithm is never stopped.
Since the (G)SEMO creates exactly one individual in each iteration and creates a single individual initially, the runtime is one plus the number of iterations until the population covers the Pareto front of~$f$ for the first time.

\subsection{The \cocz Benchmark}
\label{sec:preliminaries:benchmarks}


The function \textsc{CountingOnesCountingZeros}  (\cocz) \cite{LaumannsTZWD02,LaumannsTZ04} is defined for even problem sizes $n \in \N_{\geq 2}$. For all $x \in \{0, 1\}^n$, we have
\begin{align}
  \label{eq:cocz}
  \cocz(x) & = \bigg(\sum_{i=1}^n x_i, \sum_{i = 1}^{[n/2]} x_i + \sum_{i=n/2+1}^n (1-x_i)\bigg).
\end{align}
This popular benchmark models common goals (maximizing the number of ones in the first half of the bit-string) and conflicting goals (maximizing the number of ones resp. zeros in the second half).
Formally, let $g_1, g_2\colon \{0, 1\}^n \to \R$ denote the number of ones in the first and in the second half of the bit string, respectively.
Then, for all $x \in \{0, 1\}^n$,
\begin{align*}
  \cocz(x) & = \bigl(g_1(x)+g_2(x),g_1(x)+n/2-g_2(x)\bigr).
\end{align*}
With this notation, it is immediate that the objective space of the \cocz problem is $\cocz(\{0,1\}^n) = \{(i+j,i+n/2-j) \mid i,j \in [0..n/2]\}$. Only the objective values with $g_1(x) = n/2$ are Pareto optimal, that is, the Pareto front is $\{(n/2 + j, n - j) \mid j \in [0 .. n/2]\}$ and has size $n/2 + 1$. Most objective values, namely all with $g_1(x) \in [0..n/2-1]$, and hence most individuals, are not Pareto optima(l). This is a notable difference to benchmarks such as OneMinMax and \ojzj, where all or most individuals are Pareto-optimal and, in particular, random individuals with high probability lie on the Pareto front.

For \cocz, we finally note that for all $i \in [0 .. n/2]$, individuals with exactly~$i$ ones in their first half are either incomparable or have the same objective value.
We use this property in our proofs in \Cref{sec:theory-results}.

\subsection{Mathematical Tools}
\label{sec:preliminaries:math-tools}

In our analysis, we are mostly concerned with bounding the tails of stopping times.
To this end, we decompose a stopping time into smaller parts, each of which denotes a certain phase of the entire process.
\Cref{thm:geometric-sum-bound} provides us with strong guarantees when understanding the separate phases well.

\begin{theorem}[{\protect\cite{Witt14}}]
  \label{thm:geometric-sum-bound}
  Let $k \in \N_{\geq 1}$, and let $\{D_i\}_{i \in [k]}$ be independent geometric random variables with respective positive success probabilities $(p_i)_{i \in [k]}$.
  Let $T^{\star} \coloneqq \sum_{i \in [k]} D_i$ , $s \coloneqq \sum_{i \in [k]} \frac{1}{p_i^2}$, and $p_{\min} \coloneqq \min \{p_i \mid i \in [k]\}$.
  Then for all $\lambda \in \R_{\geq 0}$, we have
  \begin{align*}
    \Pr\bigl[T^{\star} \geq \E[T^{\star}] + \lambda\bigr] & \leq \exp\bigl(-\tfrac{1}{4}\min\bigl\{\tfrac{\lambda^2}{s}, \lambda p_{\min}\bigr\}\bigr) \textrm{ and} \\
    \Pr\bigl[T^{\star} \leq \E[T^{\star}] - \lambda\bigr] & \leq \exp\bigl(-\tfrac{\lambda^2}{2 s}\bigr) .
  \end{align*}
\end{theorem}

In order to conveniently estimate the sums appearing in applications of \Cref{thm:geometric-sum-bound}, we use the following well-known estimates.

\begin{theorem}[{\protect\cite[Inequality~(A.12)]{CormenLRS01IntroductionToAlgorithms}}]
  \label{thm:sums-to-integrals}
  Let $g\colon \R \to \R$ be a monotonically non-increasing function, and let $\alpha, \beta \in \R$ with $\alpha \leq \beta$. Then
  \begin{align*}
    \int_{\alpha}^{\beta + 1} g(x) \mathrm{d} x \leq \sum\nolimits_{x = \alpha}^{\beta} g(x)
    \leq \int_{\alpha-1}^{\beta} g(x) \mathrm{d} x.
  \end{align*}
\end{theorem}

Last, the following classic Chernoff bound is used to estimate the objective values of initial solutions.

\begin{theorem}[{\protect\cite{Chernoff52}}]
  \label{thm:chernoff-lower-bound}
  Let $k \in \N_{\geq 1}$, and let $\{X_i\}_{i \in [k]}$ be independent random variables taking values in $[0, 1]$.
  Let $X^{\star} = \sum_{i \in [k]} X_i$ and $\delta \in [0, 1]$.
  Then
  \begin{equation*}
    \Pr\bigl[X^{\star} \leq (1 - \delta) \E[X^{\star}]\bigr] \leq \exp\left(-\frac{\delta^2 \E[X]}{2}\right) .
  \end{equation*}
\end{theorem}

\section{Runtime Analysis on COCZ}
\label{sec:theory-results}

Our main result is \Cref{thm:gsemo-cocz-lower-bound} below, which proves that the (G)SEMO (\Cref{algo:GSEMO}) optimizes the \cocz benchmark (\cref{eq:cocz}) with high probability and thus also in expectation in $\Omega(n^2 \log n)$ objective-function evaluations.
This matches the $O(n^2 \log n)$ upper bound by \cite{LaumannsTZ04}, resulting overall in a tight runtime bound of $\Theta(n^2 \log n)$ expected objective-function evaluations.

\begin{theorem}
  \label{thm:gsemo-cocz-lower-bound}
  With probability $1 - \Theta(n^{-1})$, the (G)SEMO maximizes \cocz in $\Omega(n^2 \log n)$ objective-function evaluations.
\end{theorem}

Another interesting result of our analysis detailed in the following is that the (G)SEMO achieves a linear population size on the Pareto front of \cocz with high probability after only $O(n^2)$ iterations (\Cref{cor:time-to-spread}).
Previously, this time was estimated pessimistically only as $O(n^2 \log n)$.

In order to prove \Cref{thm:gsemo-cocz-lower-bound}, we need to closely follow the population size of the (G)SEMO during the run.
Although the (G)SEMO only creates a single offspring each iteration (and thus only evaluates the objective function a single time), the population size and its composition affect the algorithm's runtime crucially.
If the population size is large, progress is only made quickly if the probability is high to select a parent that can be turned likely into a useful offspring.
This probability, in turn, relies on where the current population is.
For \cocz, assume that the entire population is already on the Pareto front, that is, for each individual~$x$ in the population, we have $g_1(x) = n/2$.
If the $g_2$-values of the population consist of a contiguous interval, that is, there is an $i \in [n/2]$ such that for each $j \in [i, n/2 - i]$, there is an~$x$ in the population such that $g_2(x) = j$, then new solutions are only created likely if individuals with a $g_2$-value close to the interval borders are chosen.\footnote{For the SEMO, even only the two extreme individuals with a $g_2$-value of~$i$ or $n/2 - i$ can create a novel objective values.}
The situation is different if we assume that the individuals have more spread-out $g_2$-values, that is, if there are some $g_2$-values in the interval $[i, n/2 - i]$ that are not covered by the current population.
Then, each individual that is close to the border of some $g_2$ interval can be useful for finding novel objective values.

A central question to proving our main result (\Cref{thm:gsemo-cocz-lower-bound}) is which composition the population of the (G)SEMO has once it reaches the Pareto front (and a short time thereafter).
In order to answer this question satisfactorily, we view the progress of the algorithm covering the Pareto front of \cocz in two dimensions, namely, with respect to the maximum $g_1$-value in the population and with respect to the extremal $g_2$-values in the population.
The~$g_1$ quantity translates to how close the population is to reaching the Pareto front, as each individual with a $g_1$-value of~$n/2$ is Pareto-optimal.
The~$g_2$ quantity translates to how close the population is to reaching the values in the second objective that are hardest to achieve, that is,~$n/2$ (having only zeros in the second half) and~$n$ (having only ones in the second half).

In our analysis, we optimistically assume, roughly, that once an individual reaches the Pareto front, all other individuals are also placed there immediately by setting their number of ones in the first half to the maximum value of~$n/2$.\footnote{Actually, we place all individuals on the Pareto front after a time that is a bit longer than it takes the algorithm to reach the Pareto front, but the main idea remains the same.}
Thus, tracking the extremal $g_2$-values tells us from this moment how close the algorithm is to covering the entire Pareto front.
In a nutshell, we show that once the algorithm reaches the Pareto front, the extremal $g_2$-values are still at least order~$\sqrt{n}$ away from the borders of the~$g_2$ interval (\Cref{lem:time-for-expanding-toward-the-PF-borders}).
From there on, based on a coupon collector argument, the algorithm still requires order $n \log n$ iterations \emph{with useful parents} in order to cover the entire Pareto front.
Since we also prove that the (G)SEMO has a population size of at least~$\frac{n}{4}$ once it reaches the Pareto front (\Cref{cor:time-to-spread}), the probability to choose a useful parent is in the order~$\frac{1}{n}$.
Thus, it still takes $\Omega(n^2 \log n)$ iterations until the algorithm covers the entire front (\Cref{lem:idealized-process-lower-bound}).
Reaching the Pareto front is done faster than that (\Cref{lem:time-to-spread}), and thus \Cref{thm:gsemo-cocz-lower-bound} follows.


\textbf{A modified (G)SEMO algorithm.}
A main challenge in our proof strategy is to track the exact population size of the algorithm while there are individuals not on the Pareto front.
This is due to such individuals being potentially dominated by better solutions and then removed.
Once the entire population is on the Pareto front, this problem vanishes, as solutions either have the same objective value or are incomparable.
In order to estimate the population size more easily until the Pareto front is reached, we make the following important observation:
We only aim to show that the extremal $g_2$-values are sufficiently far from~$n/2$ and~$n$ \emph{before} the (G)SEMO reaches the Pareto front.
This is a \emph{relative} statement, essentially comparing the progress made with the maximum $g_1$-value in the population to the progress made with the extremal $g_2$-values.
Thus, we can arbitrarily modify the (G)SEMO as long as we make sure that this relative order is not harmed.
We call the resulting algorithm the \emph{modified (G)SEMO}.



The modified (G)SEMO is identical to the (G)SEMO (\Cref{algo:GSEMO}) except for line~\ref{line:selectParent}, which is replaced by the following procedure, using the notation of the pseudocode.
Choose a value $i \in [0 .. \frac{n}{2}]$ uniformly at random.
Check if $\{z \in P^{(t)} \mid g_2(z) = i\}$ is empty.
If it is, continue with the next iteration.
Otherwise, note that the set contains exactly one individual~$x^{(t)}$, as all individuals with equal $g_2$-value are comparable and~$P^{(t)}$ thus contains at most one such individual.
Continue with line~\ref{line:mutation} exactly seen in \Cref{algo:GSEMO}, using~$x^{(t)}$.
Note that the resulting modified (G)SEMO resembles a version of the original (G)SEMO that may add some pointless iterations not modifying the algorithm's state.
Thus, in particular, each upper bound on the runtime of the modified (G)SEMO is also an upper bound on the runtime of the original (G)SEMO.

A key observation is that if we consider a run of the modified (G)SEMO and remove all iterations in which an index~$i$ with no corresponding individuals is chosen, the algorithm is identical of the original (G)SEMO.
Thus, any statements about the states of either algorithm based on stopping times defined only on algorithm states are identical.
This allows us to translate results from the modified (G)SEMO to the original one, and it addresses the challenge above of estimating the population size of the original (G)SEMO very closely.
Once the modified (G)SEMO reaches the Pareto front and has a linear population size, we switch back to the original (G)SEMO in order to derive a runtime bound for this exact algorithm.



As outlined above, we compare the time it takes the modified (G)SEMO to reach the Pareto front (measured via the maximum $g_1$-value in the population) and the time it takes to reach extremal $g_2$-values in the order of~$\sqrt{n}$.
More specifically, we show that the modified (G)SEMO reaches the Pareto front with high probability in $O(n^2)$ iterations (\Cref{lem:time-to-Pareto-front}), whereas it takes $\Omega(n^2 \log n)$ iterations until the $g_2$-values progressed sufficiently far (\Cref{lem:time-for-expanding-toward-the-PF-borders}).
In addition, we show that once the modified (G)SEMO reaches the Pareto front, it reaches a population size linear in~$n$ within the same order of time (\Cref{lem:time-to-spread}).
Combining these statements, we get with high probability that the modified (G)SEMO (and thus also the original (G)SEMO) has a linear population size before the extremal $g_2$-values are close to covering the entire interval.

We recall that all upper bounds on the iterations for the modified (G)SEMO algorithm also hold for the original (G)SEMO algorithm.

\paragraph{Progress of the modified (G)SEMO on the $g_1$-values.}
We begin by showing that the modified (G)SEMO quickly reaches the Pareto front of \cocz and expands its population to a linear size (where we recall that for the (G)SEMO, different Pareto optima in the population necessarily have different objective values). To this end, we determine the probability to cover a fitness vector of the current best cooperative level if a linear fraction of these vectors is still uncovered.\footnote{For reasons of space, most proofs had to be omitted in this extended abstract. The reviewers can find them in the appendix. After the reviewing process, we will post a complete version with all proofs on the arxiv preprint server.}

\begin{restatable}{lemma}{probabilityToSpread}
  \label{lem:probability-to-spread}
  Let $0<\delta<1$. Consider one iteration of the modified (G)SEMO maximizing $f\coloneqq\cocz$, and denote by $Z_t$ the number of individuals with a maximum $g_1$-value $\ell$ in $P^{(t)}$. Suppose that $1 \leq Z_t < \delta n/2$. Then the probability to increase $Z_t$ by one is at least $\frac{1}{n/2+1} \cdot \frac{1-\delta}{4e}$.
\end{restatable}

With \Cref{lem:probability-to-spread}, we bound the expected time to find a Pareto optimal individual in the modified (G)SEMO from above.

\begin{restatable}{lemma}{timeToParetoFront}
  \label{lem:time-to-Pareto-front}
  Consider the modified (G)SEMO maximizing $f\coloneqq\cocz$.  With probability $1-\exp(-\Omega(\sqrt{n}))$, after at most $O(n^2)$ iterations, for every initialization of $x^{(0)}$ the population of the modified (G)SEMO reaches the Pareto front, i.e. $P^{(t)}$ contains a Pareto optimal individual $x$.
\end{restatable}

Once the modified (G)SEMO reaches the Pareto front, we show that it achieves a population size linear in the problem size in the same amount of time.

\begin{restatable}{lemma}{timeToSpread}
  \label{lem:time-to-spread}
  Consider the modified (G)SEMO maximizing \cocz and suppose there is an individual on the Pareto front. Then with probability $1-\exp(-\Omega(n))$, after at most $O(n^2)$ iterations the population of the modified (G)SEMO contains at least~$\frac{n}{4}$ individuals on the Pareto front.
\end{restatable}

Combining~\Cref{lem:time-to-Pareto-front,lem:time-to-spread}, we obtain that the modified (G)SEMO has a linear population size and is on the Pareto front in at most $O(n^2)$ iterations, with high probability.

\begin{corollary}
  \label{cor:time-to-spread}
  Consider the modified (G)SEMO maximizing \cocz. Then with probability $1-\exp(-\Omega(\sqrt{n}))$, after at most $O(n^2)$ iterations the population of the modified (G)SEMO contains at least~$\frac{n}{4}$ individuals on the Pareto front.
\end{corollary}

\paragraph{Progress of the modified (G)SEMO on the $g_2$-values.}
We show that the modified (G)SEMO takes some time in order to find solutions that are close to the extremal solutions of the Pareto front.
We call this value the \emph{distance to the Pareto borders}.
Formally, for all $z \in \{0, 1\}^n$, let the \emph{distance of~$z$ to the Pareto borders} be $d_{\mathrm{PF}}(z) \coloneqq \min \bigl\{g_2(z), \frac{n}{2} - g_2(z)\bigr\}$.
Using the notation of \Cref{algo:GSEMO}, for all $t \in \N$, we define the \emph{distance of the algorithm to the Pareto borders in iteration~$t$} as $d_{\mathrm{PF}}(P^{(t)}) \coloneqq \min_{z \in P^{(t)}} d_{\mathrm{PF}}(z)$.

\begin{restatable}{lemma}{timeForExpandingTowardThePFBorders}
  \label{lem:time-for-expanding-toward-the-PF-borders}
  Consider the modified (G)SEMO maximizing \cocz.
  Let $c \in \R_{> 0}$ be a sufficiently small constant.
  With probability at least $1 - \Theta(n^{-2/5})$, for all iterations $t \in [0 .. c n^2 \ln n]$, the distance of the algorithm to the Pareto borders in iteration~$t$ is at least $\sqrt{n}$.
\end{restatable}

\paragraph{How the original (G)SEMO computes the full Pareto front.}
We show that if the original (G)SEMO is started in a state that the modified (G)SEMO reaches with high probability in $O(n^2)$ iterations, the original (G)SEMO still requires at least order $n^2 \ln n$ iterations in order to cover the Pareto front.
This statement relies on the linear lower bound on the population size from \Cref{cor:time-to-spread} as well as on the distance of at least~$\sqrt{n}$ to the Pareto borders from \Cref{lem:time-for-expanding-toward-the-PF-borders}.

\begin{restatable}{lemma}{idealizedProcessLowerBound}
  \label{lem:idealized-process-lower-bound}
  Consider the (G)SEMO maximizing \cocz, starting with a population size of $\Theta(n)$ on the Pareto front and a distance to the Pareto borders of at least $\sqrt{n}$.
  Then with probability $1 - \Theta(n^{-1})$, the algorithm covers the Pareto front after $\Omega(n^2 \log n)$ objective-function evaluations.
\end{restatable}

Last, we prove our main result (\Cref{thm:gsemo-cocz-lower-bound}) by showing that the starting state assumed in \Cref{lem:idealized-process-lower-bound} is reached with high probability, as sketched before the lemma.

\begin{proof}[Proof of \Cref{thm:gsemo-cocz-lower-bound}]
  We only start counting function evaluations once the (G)SEMO has at least $\Theta(n)$ individuals on the Pareto front.
  Let~$T$ denote the first iteration in which this is the case.
  We proceed to argue why it has with probability $1 - \Theta(n^{-1})$ a distance of at least~$\sqrt{n}$ to the Pareto borders in iteration~$T$.
  By applying \Cref{lem:idealized-process-lower-bound}, the statement follows then and the proof is concluded.

  In order to show that the (G)SEMO is in the desired state in iteration~$T$, we consider the modified (G)SEMO instead.
  Recall that the original (G)SEMO changes states if and only if the modified (G)SEMO does so, albeit in potentially different iterations, and they transition into identical states.
  Let~$S$ denote the first iteration in which the modified (G)SEMO has at least $\Theta(n)$ individuals on the Pareto front.
  By \Cref{lem:time-to-Pareto-front}, with probability $1 - \exp\bigl(-\Omega(\sqrt{n})\bigr)$, we have that $S = O(n^2)$.
  Moreover, by \Cref{lem:time-for-expanding-toward-the-PF-borders}, we have with probability $1 - \Theta(n^{-2/5})$ that the distance of the modified (G)SEMO to the Pareto borders is at least~$\sqrt{n}$.
  Hence, with probability $1 - \Theta(n^{-2/5})$, the modified (G)SEMO is in the desired state in iteration~$S$.

  Since~$S$ and~$T$ refer to the same state of the respective algorithm, it follows that the original (G)SEMO is also in the desired state in iteration~$T$ with probability $1 - \Theta(n^{-1})$, concluding the proof.
\end{proof}

\section{Runtime Analysis on OMM and OJZJ}
\label{sec:additional-run-time-results}

We show that our insights from \Cref{sec:theory-results} about the population dynamics on \cocz also translate to the popular bi-objective benchmarks \omm~\cite{GielL10} and \ojzj~\cite{DoerrQ22ppsn}.

The \omm benchmark aims at maximizing and minimizing the number of ones in a bit string, resulting in \emph{all} individuals being Pareto-optimal.
This function resembles \cocz without the cooperative part. Formally, for all $x \in \{0, 1\}^n$,
\begin{align}
  \label{eq:omm}
  \omm(x)= \left(\textstyle\sum_{i\in [n]}x_i, \sum_{i\in [n]}(1-x_i)\right).
\end{align}
The Pareto front of \omm is $\{(i, n - i) \mid i \in [0 .. n]\}$.
In particular, each individual is Pareto-optimal, different from \cocz.

The \ojzj benchmark requires a \emph{gap size} $k \in [2 .. n]$ and is structurally identical to \omm for all individuals whose number of ones is at least~$k$ at most $n - k$.
Those individuals are all Pareto-optimal.
In addition, the all-ones and the all-zeros bit string are Pareto-optimal as well.
All other individuals are strictly worse.
This usually requires algorithms to flip at least~$k$ bits in order to find the extremal Pareto optima. This is formally defined as for all $x \in \{0, 1\}^n$, letting~$|x|_1$ and~$|x|_0$ denote respectively the number of ones and the number of zeros in~$x$, let $\ojzjk(x)=\bigl(f_1(x), f_2(x)\bigr)$ with
\begin{align}
  \label{eq:ojzj}
  f_1(x)=\begin{cases}
           k+|x|_1, & \text{ if $|x|_1 \leq n-k$ or $x=1^n$,} \\
           n-|x|_1, & \text{ else, and}
         \end{cases} \\
  \notag
  f_2(x)=\begin{cases}
           k+|x|_0, & \text{ if $|x|_0 \leq n-k$ or $x=0^n$,} \\
           n-|x|_0, & \text{ else.}
         \end{cases}
\end{align}
The first objective is the single-objective $\textsc{Jump}_k$ benchmark, which features a local optimum at $n - k$.
The second objective is structurally identical to the first but with the roles of ones and zeros reversed. \cite{DoerrZ21aaai} showed that the Pareto front~$F^*$ of $\ojzjk$ is $\{(a,2k+n-a) \mid a \in [2k \ldots n] \cup \{k,n-k\}\}$.
Note that each individual $x$ with $f(x) \in F^*$ strictly dominates each individual~$y$ with $f(y) \notin F^*$ since, for all $j \in [2]$, we have $f_j(x) \geq k$ but $f_j(y) \leq n-(n-k+1)=k-1$. For \omm, tight bounds are already known (see \Cref{sec:additional-run-time-results:omm}).
Hence, our results just provide a different angle of proving them.
For \ojzj, tight bounds were known for all $k \in \N_{\geq 4}$ (see \Cref{sec:additional-run-time-results:ojzj}), as a pessimistic bound of~$n$ for the population size is sufficient.
Our result shows that this bound also holds for the cases $k \in \{2, 3\}$, where our insights into the population dynamics are important.



For both benchmarks, we follow a similar strategy as in \Cref{sec:theory-results}.
Especially, we rely again on the modified (G)SEMO.
This modification needs to be slightly adjusted as follows, using the notation from its original definition:
We choose a value $i \in [0 .. n]$ uniformly at random (instead from $[0 .. \frac{n}{2}]$) and check whether the set $\{z \in P^{(t)} \mid g_1(z) + g_2(z) = i\}$ is empty.
That is, instead of focusing only on the number of ones in the first half, we now consider the number of ones in the entire bit string.
The rest remains identical.

\paragraph{Progress of the modified (G)SEMO.}
\Cref{lem:time-to-Pareto-front-OMM-OJZJ} below essentially translates \Cref{cor:time-to-spread} to \omm and \ojzj and shows that the modified (G)SEMO reaches a population size of~$\frac{n}{2}$ in $O(n^2)$ iterations.

\begin{restatable}{lemma}{timeToParetoFrontOMMOJZJ}
  \label{lem:time-to-Pareto-front-OMM-OJZJ}
  Consider the modified (G)SEMO maximizing \omm or $\ojzjk$ for $1<k\leq n/4$. With probability $1-\exp(-\Omega(\sqrt{n}))$, after at most $O(n^2)$ iterations, for every initialization of $x^{(0)}$ in case of \omm or for an initialization on the Pareto front distinct from $0^n$ and $1^n$ in case of $\ojzjk$, the population of the modified (G)SEMO contains at least~$n/2$ individuals.
\end{restatable}

\Cref{lem:time-for-expanding-toward-the-PF-borders-OMM-OJZJ} below translates \Cref{lem:time-for-expanding-toward-the-PF-borders} to this setting and also uses its terminology.
It shows that the extremal solutions in the population are still at least~$\sqrt{n}$ away from the borders.
We need to re-define though what these terms exactly mean in the setting of \omm and \ojzj.

The extremal solutions of the Pareto front are $1^n$ and $0^n$ instead of $1^n$ and $1^{n/2}0^{n/2}$ for \cocz. Furthermore, for all $z \in \{0, 1\}^n$, let the \emph{distance of~$z$ to the Pareto borders} be $d_{\mathrm{PF}}(z) \coloneqq \min \bigl\{\ones{z}, n - \ones{z}\bigr\}$ which is also slightly different to the case of \cocz above.
However, the \emph{distance of the algorithm to the Pareto borders in iteration~$t$} is defined as $d_{\mathrm{PF}}(P^{(t)}) \coloneqq \min_{z \in P^{(t)}} d_{\mathrm{PF}}(z)$ for all $t \in \mathbb{N}$ in the same way as in \Cref{sec:theory-results}.

\begin{restatable}{lemma}{timeForExpandingTowardThePFBordersOMMOJZJ}
  \label{lem:time-for-expanding-toward-the-PF-borders-OMM-OJZJ}
  Consider the modified (G)SEMO maximizing \omm or $\ojzjk$ for $1 < k \leq n/4$. Let $c \in \R_{> 0}$ be a sufficiently small constant.
  With probability at least $1 - \Theta(n^{-2/5})$, for all iterations $t \in [0 .. c n^2 \ln n]$, the distance of the algorithm to the Pareto borders in iteration~$t$ is at least $\sqrt{n}$ for \omm and at least $\max\{\sqrt{n}, k\}$ for $\ojzjk$.
\end{restatable}

\subsection{\omm}
\label{sec:additional-run-time-results:omm}

For \omm we prove a bound of $\Omega(n^2 \log n)$ objective-function evaluations, with high probability (\Cref{thm:gsemo-lower-bound-OMM}).
This matches the $O(n^2 \log n)$ bound by \cite{GielL10}.
The tight $\Theta(n^2 \log n)$ runtime for the GSEMO was already proven by \cite{BossekS24} as a special case of the single-objective problem of quality diversity on the \textsc{OneMax} benchmark.
The bound $\Theta(n^2 \log n)$ for the SEMO was shown by \cite{OsunaGNS20}.

\begin{restatable}{theorem}{gsemoLowerBoundOMM}
  \label{thm:gsemo-lower-bound-OMM}
  With probability $1 - \Theta(n^{-1})$, the (G)SEMO maximizes \omm in $\Omega(n^2 \log n)$ objective-function evaluations.
\end{restatable}

The proof of \Cref{thm:gsemo-lower-bound-OMM} is very similar to that of \Cref{thm:gsemo-cocz-lower-bound}. We use \Cref{lem:time-to-Pareto-front-OMM-OJZJ,lem:time-for-expanding-toward-the-PF-borders-OMM-OJZJ} for \omm from above.

\begin{restatable}{lemma}{idealizedProcessLowerBoundOMM}
  \label{lem:idealized-process-lower-bound-OMM}
  Consider the (G)SEMO maximizing \omm, starting with a population size of $\Theta(n)$ on the Pareto front and the distance to the Pareto borders is at least $\sqrt{n}$. 
  Then with probability $1 - \Theta(n^{-1})$, the algorithm covers the Pareto front after $\Omega(n^2 \log n)$ objective-function evaluations.
\end{restatable}

By combining \Cref{lem:time-to-Pareto-front-OMM-OJZJ,lem:time-for-expanding-toward-the-PF-borders-OMM-OJZJ,lem:idealized-process-lower-bound-OMM}, we obtain the proof for \Cref{thm:gsemo-lower-bound-OMM} in a similar way as the proof of \Cref{thm:gsemo-cocz-lower-bound}.

\subsection{\ojzj}
\label{sec:additional-run-time-results:ojzj}

For \ojzj, we only consider the GSEMO, as the SEMO does not cover the Pareto front with high probability (\cite{DoerrZ21aaai}), due to the deceptive nature of the benchmark and the 1-bit mutation being incapable of creating the extremal Pareto optima from non-extremal Pareto optima.
For gap size $k \in [2 .. \frac{n}{4}]$, we prove a bound of $\Omega(n^{k+1})$ objective-function evaluations, with high probability (thm:gsemo-lower-bound-OJZJ).
This matches the bound $O(n^{k + 1})$ for all of these values of~$k$ by \cite{DoerrZ21aaai}.
Moreover, for $k \in [4 .. \frac{n}{2} - 1]$, \cite{DoerrZ21aaai} proved already a matching bound of $\Omega((n - 2k)n^k)$.
However, for $k \in \{2, 3\}$, our results are new.

\begin{restatable}{theorem}{gsemoLowerBoundOJZJ}
  \label{thm:gsemo-lower-bound-OJZJ}
  In expectation, the GSEMO maximizes $\ojzjk$ for $k \in [2 .. \frac{n}{4}]$ in $\Omega(n^{k+1})$ objective-function evaluations.
\end{restatable}

The proof of \Cref{thm:gsemo-lower-bound-OJZJ} makes use of \Cref{lem:idealized-process-lower-bound-OJZJ} below, which shows that it takes the GSEMO a lot of time to find the all-ones and all-zeros bit string.

\begin{restatable}{lemma}{idealizedProcessLowerBoundOJZJ}
  \label{lem:idealized-process-lower-bound-OJZJ}
  Consider the GSEMO maximizing $\ojzjk$ for $1 < k \leq n/4$, starting with a population size of $\Theta(n)$ on the Pareto front, but neither $0^n$ nor $1^n$ are in the population. 
  Then the algorithm covers the Pareto front in expectation after $\Omega(n^{k+1})$ objective-function evaluations.
\end{restatable}

\Cref{thm:gsemo-lower-bound-OJZJ} can be proven in a completely similar way as \Cref{thm:gsemo-cocz-lower-bound,thm:gsemo-lower-bound-OMM} by using \Cref{lem:time-to-Pareto-front-OMM-OJZJ,lem:time-for-expanding-toward-the-PF-borders-OMM-OJZJ} for $\ojzjk$ when $1<k\leq n/4$ and then following \Cref{lem:idealized-process-lower-bound-OJZJ}.

\section{Conclusion}
\label{sec:conclusion}

We studied the population dynamics of the (G)SEMO, that is, the size and shape of its population over time.
For the \cocz benchmark, we proved that the algorithm has a population size linear in the size of the Pareto front (\Cref{cor:time-to-spread}) while it is still sufficiently far away from covering the entire Pareto front (\Cref{lem:time-for-expanding-toward-the-PF-borders}).
Covering these remaining solutions takes at least $\Omega(n^2 \log n)$ iterations (\Cref{thm:gsemo-cocz-lower-bound}).
Since a matching upper bound exists, this result is tight.

Our proof strategy relies on defining a modified process that allows an easier estimate of the probability to select a useful parent for making progress.
This modification affects the absolute number of iterations but not the relative order of state changes.
Thus, insights into state behavior for the modification also translate directly back to the original (G)SEMO.
We believe that this is an interesting strategy that may be useful for other settings with dynamic quantities.

We show that our insights for \cocz also transfer to \omm and \ojzj, where we derive lower bounds that match known upper bounds.
Although most of these lower bounds were already known, our proof strategy provides a different angle for deriving them.
Moreover, our lower bounds for gap sizes $k \in \{2, 3\}$ for $\ojzjk$ are new.
Since we prove a general lower bound for a large range of~$k$, all of which are tight, our method captures the true nature of the (G)SEMO well.

For future work, it would be interesting to derive lower bounds for the population size while the (G)SEMO is approaching the Pareto front.
In this article, we only derive upper bounds for this time (\Cref{cor:time-to-spread}).
Lower bounds would give us a clearer picture how wasteful the algorithm is in terms of function evaluations before reaching the Pareto front.

Another interesting open problem is to derive tight lower bounds of the GSEMO for the \textsc{LeadingOnesTrailingZeros} (LOTZ) benchmark.
For this setting, so far, only the bound $O(n^3)$ \cite{Giel03} exists (but a $\Theta(n^3)$ bound for the SEMO).
In contrast to the setting in this article, for LOTZ, it is far more likely for the GSEMO to get closer to the Pareto front than to create an incomparable offspring.
Only once the Pareto front is reached, these two probabilities are in the same order.
Hence, the dynamics seem to be somewhat different from the cases we consider here.

Our result can also serve as a stepping stone towards the deeper understanding of the population dynamics of other MOEAs, like the NSGA-II. 
However, since the NSGA-II has a fixed population size,
parts of the analysis need to focus on how duplicate entries in the population are treated, which is a separate topic and thus left for future work.
}

\bibliographystyle{named}
\bibliography{ich_master,alles_ea_master,rest}

\cleardoublepage
\appendix
\onecolumn

\section*{Supplementary Material for Paper 5678 -- ``Understanding the Population Dynamics of the (G)SEMO: A Tight Lower Runtime Bound on the COCZ Benchmark''}

This document contains the proofs that we omitted in the main paper, due to space restrictions.
It is meant to be read at the reviewer's discretion only.

For the sake of convenience, the statements are restated with the same number that they have in the main paper.

\section{Runtime Analysis on COCZ}
\label{sec:appendix:coczAnalysis}

\probabilityToSpread*

\begin{proof}
  Consider the following cases:
  \begin{itemize}
    \item[(1)]
          Suppose that there is an individual $x' \in P^{(t)}$ with $\onesa{x'}=\ell$ and $\onesb{x'} < n/4-\delta n/4$. Since $Z_t < \delta n/2$, one finds $k \in [0,\delta n/2]$ such that there is no individual $y'$ with $\onesa{y'}=\ell$ and $\onesb{y'} = \lceil{n/4-\delta n/4 + k}\rceil$. Hence, there is an individual $x$ with $\onesa{x}=\ell$ and $\onesb{x} <\lceil{n/4+\delta n/4}\rceil$, but no individual $y$ with $\onesa{y}=\ell$ and $\onesb{y} = \onesb{x} + 1$.
    \item[(2)]
          Suppose that there is no individual $x' \in P^{(t)}$ with $\onesa{x'}=\ell$ and $\onesb{x'} < n/4-\delta n/4$. Then one finds an individual $x \in P^{(t)}$ with $\onesa{x}=\ell$ and $\onesb{x} \geq n/4-\delta n/4$, but no $y \in P^{(t)}$ with $\onesa{y}=\ell$ and $\onesb{y} = \onesb{x} - 1$.
  \end{itemize}
  To increase $Z_t$, one may choose $x$ as parent and create $y$ by flipping one of at least $\lceil{n/4-\delta n/4}\rceil$ many bits and keeping the remaining bits unchanged. For both modified algorithms, this happens with probability at least
  \begin{align*}
     & \frac{1}{n/2+1} \cdot \frac{\ceil{n/4-\delta n/4}}{n} \cdot \left(1-\frac{1}{n}\right)^{n-1} \geq \frac{1}{n/2+1} \cdot \frac{n/4-\delta n/4}{ne} = \frac{1}{n/2+1}\cdot \left(\frac{1-\delta}{4e}\right). \qedhere 
  \end{align*}
\end{proof}

\timeToParetoFront*

\begin{proof}
  Denote by the random variable $T$ the number of iterations to create a first Pareto optimal individual. We define the process $X_t \coloneqq n/2-\max_{x \in P^{(t)}} \onesa{x}$ and a second process $Z_t \coloneqq |\{x \in P^{(t)} \mid \onesa{x}=n/2-X_t\}|$.
  Note that $X_t$ is the smallest possible number of zeros in the first half of an individual from the current population $P^{(t)}$, and $Z_t$ counts the number of those individuals in $P^{(t)}$. Note that $X_t$ cannot increase since an individual $y$ with a maximum number of ones in its first half has a maximum sum of both objective values (which is $n/2+2g_1(y)$) and hence, is non-dominated. Moreover, $Z_t$ can only decrease if $X_t$ decreases, and one reaches the Pareto front if $X_t$ becomes zero. We first show that it takes an expected number of $O(n)$ iterations to increase $Z_t$ to $\sqrt{n/X_t}$ or to decrease $X_t$. Note that if $X_t$ decreases then $Z_t$ decreases to one.  If $Z_t < \sqrt{n/X_t}$ then $Z_t$ is increased with probability at least  $\frac{1}{n/2+1} \cdot \frac{1}{5e}$ by Lemma~\ref{lem:probability-to-spread} applied to $\delta=1/5$ (for $n$ sufficiently large). Now, suppose that $Z_t \geq \sqrt{n/X_t}$. Then, to decrease $X_t$, one may choose such an individual as a parent and flip one of the $X_t$ zeros in the first half of the bit string. For both modified algorithms, this happens with probability at least
  \[
    \frac{\ceil{\sqrt{n/X_t}}}{n/2+1} \cdot \frac{X_t}{en} \geq \frac{\sqrt{X_t}}{en^{1.5}}.
  \]
  Note that $n/2$ such decrements of $X_t$ are sufficient to reach the Pareto front and if $X_t$ increases then $Z_t$ becomes $1$. Hence, we need at most
  \begin{align*}
    \sum\nolimits_{i=1}^{n/2}\left(\ceil*{\sqrt{\frac{n}{i}}}-1\right) & \leq \sum\nolimits_{i=1}^{n/2}\sqrt{\frac{n}{i}} = \sqrt{n}\sum\nolimits_{i=1}^{n/2}\sqrt{\frac{1}{i}}           \\
                                                                       & \leq \sqrt{n}\Bigl(1+\int_{1}^{n/2} \frac{1}{\sqrt{x}}dx\Bigr) = \sqrt{n}\Bigl(1+[2\sqrt{x}]_{x=1}^{x=n/2}\Bigr) \\
                                                                       & = \sqrt{n}(\sqrt{2n}-1) \leq 2n
  \end{align*}
  increments of $Z_t$ in total to find a first Pareto optimal point. Hence, $T$ is stochastically dominated by the sum of $2n$ independent geometrically distributed random variables $W_1, \ldots , W_{2n}$ with success probability $\frac{1}{(n/2+1)5e} \eqqcolon p$ (describing the increments of $Z_t$) and $n/2$ independent geometrically distributed random variables $Y_1, \ldots , Y_{n/2}$ (describing the decrements of $X_t$) such that $Y_i$ has success probability $\frac{\sqrt{i}}{en^{1.5}} \eqqcolon q_i$. Therefore, for $c \coloneqq 20e$, $V \coloneqq Y + W$ where $Y\coloneqq\sum\nolimits_{i=1}^{n/2} Y_i$ and $W\coloneqq\sum\nolimits_{i=1}^{2n} W_i$, we have
  \[
    \Pr[T \geq cn^2] \leq \Pr[V \geq cn^2] \leq \Pr\Bigl[\sum\nolimits_{i=1}^{n/2}Y_i \geq cn^2/2\Bigr] + \Pr\Bigl[\sum\nolimits_{i=1}^{2n}W_i \geq cn^2/2\Bigr].
  \]
  Now we estimate both summands separately. By linearity of expectation, $\E[W]= 2n \cdot (5e(n/2+1)) = 5en^2+ 10en$ and
  \[
    \E[Y] = \sum\nolimits_{i=1}^{n/2} \frac{en^{1.5}}{\sqrt{i}} = en^{1.5} \sum\nolimits_{i=1}^{n/2}\frac{1}{\sqrt{i}} \leq en^{1.5}(\sqrt{2n}-1) \leq \sqrt{2}en^2 \leq 2en^2
  \]
  where the first inequality can be derived similar as above. With $\lambda \coloneqq 5en^2-10en$, $p_{\min} = p$ and $s=\sum\nolimits_{i=1}^{2n} (1/p)^2 = 50e^2n(n/2+1)^2$, \Cref{thm:geometric-sum-bound} yields
  \begin{align*}
    \Pr\Bigl[W \geq cn^2/2\Bigr] & = \Pr\Bigl[W \geq 10en^2\Bigr] = \Pr\Bigl[W \geq \E[W] +5en^2-10en \Bigr]                                      \\
                                 & \leq \exp\left(-\frac{1}{4} \min\left\{\frac{\lambda^2}{s},\lambda p_{\min}\right\}\right) = \exp(-\Omega(n)).
  \end{align*}
  Further, we also obtain with \Cref{thm:geometric-sum-bound} that for $\lambda' \coloneqq 8en^2$, $s' \coloneqq \sum\nolimits_{i=1}^{n/2} (1/q_i)^2 = \sum\nolimits_{i=1}^{n/2}\frac{e^2n^3}{i} \leq e^2n^3(\ln(n/2)+1)$ and $q_{\min} = \frac{1}{en^{1.5}}$, we have
  \begin{align*}
    \Pr\Bigl[Y \geq cn^2/2\Bigr] = & \Pr\Bigl[Y \geq 10en^2\Bigr] = \Pr\Bigl[Y \geq \E[Y] +10en^2 - \E[Y]\Bigr]                                                             \\
                                   & \leq \Pr\Bigl[Y \geq \E[Y] +8en^2 \Bigr] \leq \exp\left(-\frac{1}{4} \min\left\{\frac{\lambda'^2}{s'},\lambda' q_{\min}\right\}\right) \\
                                   & \leq \exp\left(-\frac{1}{4} \min\left\{64n/(\ln(n/2)+1),8\sqrt{n}\right\}\right) = \exp(-\Omega(\sqrt{n})).
  \end{align*}
  Hence, we obtain
  \begin{align*}
    \Pr[T \geq 20en^2] & \leq \Pr[V \geq 20en^2] \leq \Pr[W \geq 10en^2] + \Pr[Y \geq 10en^2] \\
                       & \leq \exp(-\Omega(\sqrt{n})). \qedhere
  \end{align*}
\end{proof}

\timeToSpread*

\begin{proof}
  Let $T$ be the time until the population of the modified (G)SEMO contains at least $n/4$ individuals on the Pareto front. Denote by $S_t$ the number of Pareto optimal points in the current population $P^{(t)}$. Note that $S_t \geq 1$. As long as $S_t <  n/4$, we apply Lemma~\ref{lem:probability-to-spread} on $\delta=1/2$ and obtain that the probability is at least $r\coloneqq\frac
    {1}{n/2+1} \frac{1}{8e}$ to increase $S_t$ for both modified algorithms.
  Hence, $T_2$ is stochastically dominated by the sum of $\lceil{n/4}\rceil-1$ geometrically distributed random variables  $U_1, \ldots , U_{\lceil{n/4}\rceil-1}$ with success probability $r$. Let $U \coloneqq  \sum\nolimits_{i=1}^{\lceil{n/4}\rceil-1} U_i$. Note that $\E[U] = (\lceil{n/4}\rceil-1)(4en+8e) \leq n/4(4en+8e) = en^2 + 2en \leq 3en^2$. Now we use again \Cref{thm:geometric-sum-bound} and obtain for $\tilde{\lambda} \coloneqq en^2$, and $\tilde{s} \coloneqq (\lceil{n/4}\rceil-1)/r^2 = (\lceil{n/4}\rceil-1)(4en+8e)^2$
  \begin{align*}
    \Pr[T \geq 4en^2] & \leq \Pr[U \geq 4en^2] \leq \Pr[U \geq \E[U] + en^2]                                                                               \\
                      & \leq \exp\left(-\frac{1}{4} \min\left\{\frac{\tilde{\lambda}^2}{\tilde{s}},\tilde{\lambda} r\right\}\right) \leq \exp(-\Omega(n)),
  \end{align*}
  which proves the lemma.
\end{proof}

\timeForExpandingTowardThePFBorders*

\newcommand{\Xl}{X^{\leftarrow}}
\newcommand{\Xr}{X^{\rightarrow}}
\begin{proof}
  We consider the distance to either Pareto border separately and then conclude via a union bound over both cases.
  To this end, let $(\Xl_t)_{t \in \N}$ be such that for all $t \in \N$, we have $\Xl_t = \min_{z \in P^{(t)}} g_2(z)$, and analogously $(\Xr_t)_{t \in \N}$ such that for all $t \in \N$, we have $\Xr_t = \min_{z \in P^{(t)}} \bigl(\frac{n}{2} - g_2(z)\bigr)$.
  Note that for all $t \in \N$, we have $d_{\mathrm{PF}}(P^{(t)}) = \min \{\Xl_t, \Xr_t\}$.

  In what follows, we only analyze~$\Xl$, as the arguments for~$\Xr$ are identical when flipping all bits in the second half of the initial individual~$x^{(0)}$, which results in equally likely algorithm trajectories in either case, due to the uniform random choice of~$x^{(0)}$.

  Due to Chernoff bounds (\Cref{thm:chernoff-lower-bound}) applied to $g_2(x^{(0)}) = \nbOnes{x_{[n/2 .. n]}^{(0)}}$ with $\delta = \frac{1}{2}$, noting that $\E[g_2(x^{(0)})] = \frac{n}{4}$, we have with probability at least $1 - \exp(-\frac{n}{32})$ that $g_2(x^{(0)}) \geq \frac{n}{8}$ and thus $\Xl_0 \geq \frac{n}{8}$.
  For the rest of the proof, we implicitly condition on $\Xl_0 \geq \frac{n}{8}$.

  We make a case distinction with respect to which mutation operator we consider.
  We start with the modified SEMO, as it is simpler and since the case for the modified GSEMO follows the same outline but adds more complexity.

  \textbf{Modified SEMO.}
  We define $\floor{\frac{n}{8}} - \ceil{\sqrt{n}}$ geometric random variables $\{D_i\}_{i \in [\sqrt{n} .. n/8 - 1]}$, each of which represents the time that it takes the modified SEMO to reduce the current~$\Xl$ value.
  To this end, let $T_{\floor{n/8}} \coloneqq \inf \{t \in \N \mid \Xl_t < \floor{\frac{n}{8}}\}$.
  Moreover, for all $i \in [\floor{\frac{n}{8}} - \ceil{\sqrt{n}}]$, let $T_{\floor{n/8} - i} \coloneqq \inf \{t \in \N \mid t \geq T_{\floor{n/8} - i + 1} \land \Xl_t < \floor{\frac{n}{8}} - i\}$, and let $D_{\floor{n/8} - i} \coloneqq T_{\floor{n/8} - i} - T_{\floor{n/8} - i + 1}$.
  Note that since SEMO uses 1-bit flips, the stopping times $\{D_i\}_{i \in [\sqrt{n} .. n/8 - 1]}$ are all positive and independent.

  Let $i \in [\sqrt{n} .. n/8 - 1]$.
  Note that~$D_i$ follows a geometric distribution with a to-be-determined success probability~$p_i$, since each try to reduce the current~$\Xl$ value of~$i$ is identical and independent of any other try, as the reduction can only occur by using the (unique) individual~$z$ in the population with $g_2(z) = i$, as SEMO only uses 1-bit flips.
  The probability to choose~$z$ is by the definition of the modified SEMO exactly $\frac{1}{n/2 + 1}$.
  Given that~$z$ is chosen, the probability to reduce~$\Xl$ is given exactly by flipping one of the~$i$ $1$s, which has probability~$\frac{i}{n}$.
  Thus, $p_i = \frac{1}{n/2 + 1} \cdot \frac{i}{n}$.

  Let $T^{\star} \coloneqq \inf \{t \in \N \mid \Xl_t < \ceil{\sqrt{n}}\} - T_{\floor{n/8}}$, and note that $T^{\star} = \sum\nolimits_{i \in [\sqrt{n} .. n/8 - 1]} D_i$.
  We aim at applying \Cref{thm:geometric-sum-bound}, using the theorem's notation.
  To this end, using our exact bound for almost all success probabilities and \Cref{thm:sums-to-integrals}, we get for sufficiently large~$n$
  \begin{align*}
    \mu
     & \coloneqq\sum\nolimits_{i \in [\sqrt{n} .. n/8 - 1]} \E[D_i]
    = \sum\nolimits_{i \in [\sqrt{n} .. n/8 - 1]} \frac{1}{p_i}
    = \sum\nolimits_{i \in [\sqrt{n} .. n/8 - 1]} \frac{(n / 2 + 1) n}{i}                                                             \\
     & \geq \frac{n^2}{2} \sum\nolimits_{i \in [\sqrt{n} .. n/8 - 1]} \frac{1}{i}
    \geq \frac{n^2}{2} \int_{[\lceil \sqrt{n} \rceil, \lfloor n / 8\rfloor]} \frac{1}{i} \mathrm{d} i                                 \\
     & \geq \frac{n^2}{2} \left(\ln\left(\left\lfloor \frac{n}{8} \right\rfloor\right) - \ln \left\lceil \sqrt{n} \right\rceil\right)
    \geq \frac{n^2}{2} \ln \frac{\sqrt{n}}{16}
    \geq \frac{n^2}{64} \ln n .
  \end{align*}
  Moreover, we get
  \begin{align*}
    s
     & \coloneqq \sum\nolimits_{i \in [\sqrt{n} .. n/8 - 1]} \frac{1}{p_i^2}
    \leq \sum\nolimits_{i \in [\sqrt{n} .. n/8 - 1]} \left(\frac{(n / 2 + 1) n}{i}\right)^2
    \leq n^4 \sum\nolimits_{i \in \N_{\geq \sqrt{n}}} \frac{1}{i^2}          \\
     & \leq n^4 \int_{\sqrt{n} - 1}^{\infty} \frac{1}{i^2} \mathrm{d} i
    = n^4 \cdot (n^{-1/2} + 1)
    \leq 2 n^{7/2}.
  \end{align*}
  Thus, by \Cref{thm:geometric-sum-bound}, we get for $\lambda \coloneqq \frac{n^2}{128} \ln n \leq \frac{\mu}{2}$ that
  \begin{align*}
    \Pr[T^{\star} \leq \lambda]
    \leq \Pr[T^{\star} \leq \mu - \lambda]
    \leq \exp\left(-\frac{\lambda^2}{2 s}\right)
    = \exp\left(-\frac{1}{2 \cdot 128^2} \sqrt{n} \ln^2 n\right) .
  \end{align*}
  Together with the probability of $1 - \exp\bigl(-\Theta(n)\bigr)$ of the initial Chernoff bound for the initialization, the overall probability of~$\Xl$ being at least~$\sqrt{n}$ for at least $cn^2 \ln n$ iterations is at least $\bigl(1 - \exp\bigl(-\Theta(n)\bigr)\bigr)\bigl(1 - \exp\bigl(-\Theta(\sqrt{n} \log^2 n)\bigr)\bigr) \geq 1 - \exp\bigl(-\Omega(\sqrt{n})\bigr)$.
  This proves the claim for~$\Xl$.
  The union bound over the cases~$\Xl$ and~$\Xr$ concludes the proof for the modified SEMO.

  \textbf{Modified GSEMO.}
  We follow the same outline as in the case for the modified SEMO above.
  However, as standard bit mutation can flip more than one bit at a time, some of the stopping times we defined before may be identical and thus not independent.
  We circumvent this problem by only considering values of~$\Xl$ where it is very unlikely to make progress by flipping at least~$11$ ones within~$n^3$ iterations.
  Conditional on this event, we then proceed as in the analysis of the modified SEMO, allowing though to reduce~$\Xl$ by up to~$10$ with each improvement.

  Let $T' \coloneqq \inf \{t \in \N \mid \Xl_t < \ceil{\sqrt{n}}\} - \inf \{t \in \N \mid \Xl_t < \floor{2 n^{3/5}}\}$.
  We aim at showing that $\Pr[T' > c n^2 \ln n] \geq 1 - \exp\bigl(-\Theta(\sqrt{n} \log^2 n)\bigr)$.

  Let $t \in \N$, and assume that $\Xl_t \in [\sqrt{n} .. 2 n^{3/5}]$.
  Assuming that the algorithm actually selects an individual in iteration~$t$ (and does not immediately continue with the next iteration), the probability to reduce~$\Xl_t$ when flipping exactly $k \in [11 .. \Xl_t]$ ones is at most
  \begin{equation*}
    \binom{\Xl_t}{k} \frac{1}{n^k} \left(1 - \frac{1}{n}\right)^{X_t - k}
    \leq \left(\frac{\Xl_t}{n}\right)^k
    \leq 2 n^{-2k/5} ,
  \end{equation*}
  which is maximized for $k = 11$.
  Taking a union bound over all at most~$n$ values for~$k$, we see that the probability to make progress by flipping at least~$11$ ones is at most $2 n^{-22/5 + 1} = 2 n^{-17/5}$ during a single iteration (and for the considered range of~$\Xl$).
  Via a union bound over~$n^3$ iterations, the probability to make progress within~$n^3$ iterations by flipping at least~$11$ bits is thus at most $2 n^{-17/5 + 3} = 2 n^{-2/5}$.

  Furthermore, we show that during~$n^3$ iterations, it is very unlikely to flip at least~$\floor{n^{3/5}}$ ones.
  Similar to above, for any value of $\Xl \in [0 .. \frac{n}{2}]$, the probability to flip exactly $k \in [n^{3/5} .. n]$ ones is at most $\binom{\Xl_t}{k} \frac{1}{n^{k}} \leq (\frac{n / 2}{n})^{k} = 2^{-k}$, which is maximized for $k = \floor{n^{3/5}}$.
  Via a union bound over all at most~$n$ values for~$k$ and all~$n^3$ iterations we consider, the probability to flip at least~$\floor{n^{3/5}}$ ones once is at most $n^4 \cdot 2^{-\floor{n^{3/5}}} \leq \exp\bigl(-\Theta(n^{3/5})\bigr)$.
  In the following, we condition on never flipping at least~$\floor{n^{3/5}}$ ones within the first~$n^3$ iterations.

  Like in the case for the modified SEMO, we aim at applying \Cref{thm:geometric-sum-bound} to a well-chosen set of $J \in [\floor{2 n^{3/5}} - \ceil{\sqrt{n}}]$ independent geometric random variables, where~$J$ is a random variable.
  To this end, assume for the following definitions that once $\Xl < \floor{n^{3/5}}$,
  \begin{enumerate}[label=(\alph*)]
    \item\label[assumption]{item:never-more-than-ten-ones-flipped}
          the modified GSEMO \emph{never} reduces~$\Xl$ by flipping at least~$11$ ones, and
    \item\label[assumption]{item:fill-up-the-population}
          the population contains for each $i \in [\Xl .. \frac{n}{2}]$ an individual with $g_2$-value of exactly~$i$.
  \end{enumerate}
  Let $T_0 \coloneqq \inf \{t \in \N \mid \Xl_t < \floor{2 n^{3/5}}\}$, and for all $i \in [J]$, let $T_i \coloneqq \inf \{t \in \N \mid t \geq T_{i-1} \land \Xl_t < \Xl_{T_{i - 1}}\}$, and let $T_J \coloneqq \inf \{t \in \N \mid \Xl_t < \ceil{\sqrt{n}}\}$.
  Last, for all $i \in [J]$, let $D_i \coloneqq T_i - T_{i - 1}$.

  We make three observations:
  \begin{enumerate}
    \item\label[observation]{item:high-starting-value}
          Since we condition on never flipping at least~$\floor{n^{3/5}}$ ones within the first~$n^3$ iterations, we have that if $T_0 \leq n^3$, then $\Xl_{T_0} \geq \floor{2 n^{3/5}} - \floor{n^{3/5}} \geq \floor{n^{3/5}}$.

    \item\label[observation]{item:identical-stopping-times}
          Conditional on the modified GSEMO never reducing~$\Xl$ by flipping at least~$11$ ones during the first~$n^3$ iterations once $\Xl < \floor{n^{3/5}}$, it holds for all $t \in [0 .. n^3 - 1]$ that $\Pr[T' \leq t] \leq \Pr[\sum_{i \in [J]} D_i \leq t]$, because the modified GSEMO without \cref{item:fill-up-the-population} (referring to~$T'$) has at most as many options to reduce~$\Xl$ as the modified GSEMO with \cref{item:fill-up-the-population} within the first~$n^3$ iterations.

    \item\label[observation]{item:independence-of-the-transition-times}
          The variables $\{D_i\}_{i \in [J]}$ are independent, because when~$\Xl$ is reduced, the algorithm always transitions into a state that is independent of the time it took to get there (by \cref{item:fill-up-the-population} above).
          Thus, \Cref{thm:geometric-sum-bound} is applicable to any subset of $\{D_i\}_{i \in [J]}$ of deterministic cardinality, and the probability bound carries over to~$T'$ if we consider bounds within at most~$n^3$ iterations starting from~$T_0$.
  \end{enumerate}

  We aim at choosing a deterministic subset of~$[J]$ that maximizes the probability bound we get from \Cref{thm:geometric-sum-bound}, relying on \cref{item:independence-of-the-transition-times}.
  To this end, let $T^{\star} \coloneqq \sum_{i \in [J]} D_i$, and let $(p_i)_{i \in [J]}$ denote the success probabilities of $(T_i)_{i \in [J]}$.
  We proceed by estimating $\mu \coloneqq \E[T^{\star}]$ and $s \coloneqq \sum_{[i \in [J]]} \frac{1}{p_i^2}$ independently for all feasible deterministic subsets of~$[J]$.
  To this end, we bound for all $i \in [J]$ the success probability~$p_i$ in both directions, recalling that we never make progress by flipping more than~$11$ ones, by \cref{item:never-more-than-ten-ones-flipped}.
  For the lower bound, there is always at least one individual in the population that allows to make progress by flipping exactly~$1$ bit (by \cref{item:fill-up-the-population}) out of the~$\Xl_{T_{i - 1}}$ possible improvements.
  Thus, we have $p_i \geq \frac{1}{n/2 + 1} \cdot \bigl(\Xl_{T_{i - 1}}/n\bigr) (1 - \frac{1}{n})^{9} \geq \frac{1}{e} \cdot \frac{1}{n/2 + 1} \cdot \bigl(\Xl_{T_{i - 1}}/n\bigr)$.
  For the upper bound, there are by \cref{item:fill-up-the-population} always~$10$ individuals in the population that improve~$\Xl_{T_{i - 1}}$ if the appropriate bits are flipped, of which there can be multiple possibilities per individual.
  For our bound, we assume optimistically that for each individual, a $1$-bit flip is sufficient, but we account for it~$10$ times.
  Thus, we obtain $p_i \leq 10 \frac{10}{n/2 + 1} \bigl(\Xl_{T_{i - 1}}/n\bigr) (1 - \frac{1}{n})^{9} \leq \frac{100}{n/2 + 1} \bigl(\Xl_{T_{i - 1}}/n\bigr)$.

  For~$\mu$, we aim at a lower bound.
  Note that by \cref{item:high-starting-value}, each subset we consider needs to have a cardinality of at least $(\floor{n^{3/5}} - \ceil{\sqrt{n}}) / 10$, as the largest improvement reduces~$\Xl$ by~$10$ and as $\Xl_{T_0} \geq \floor{n^{3/5}}$.
  Thus, we get by our upper bound on~$p_i$ and by \Cref{thm:sums-to-integrals}
  \begin{align*}
    \mu
     & = \sum\nolimits_{i \in [J]} \E[D_i]
    \geq \frac{n^2}{100} \sum\nolimits_{i \in [(\floor{n^{3/5}} - \ceil{\sqrt{n}}) / 10]} \frac{1}{\ceil{\sqrt{n}} + 10i}          \\
     & \geq \frac{n^2}{100} \sum\nolimits_{i \in [\ceil{\sqrt{n}} / 10 + 1 .. \floor{n^{3/5}} / 10 - 1]} \frac{1}{10 i}            \\
     & \geq \frac{n^2}{1000} \int_{[\ceil{\sqrt{n}} / 10 + 1, \floor{n^{3/5}} / 10]} \frac{1}{i} \mathrm{d}i                       \\
     & = \frac{n^2}{1000} \left(\ln\left(\frac{\floor{n^{3/5}}}{10}\right) - \ln\left(\frac{\ceil{\sqrt{n}}}{10} + 1\right)\right) \\
     & \geq \frac{n^2}{1000} \ln \frac{\floor{n^{3/5}}}{2 \ceil{\sqrt{n}}}
    \geq \frac{n^2}{1000} \ln \frac{n^{1/10}}{4}
    \geq \frac{n^2}{20000} \ln n .
  \end{align*}
  For~$s$, we aim at an upper bound, using our lower bound on~$p_i$.
  We obtain
  \begin{align*}
    s
     & = \sum\nolimits_{i \in [J]} \frac{1}{p_i^2}
    \leq \sum\nolimits_{i \in [\sqrt{n} .. 2 n^{3/5}]} \left(e \left(\frac{n}{2} + 1\right) \frac{n}{i}\right)^2
    \leq e^2 n^4 \sum\nolimits_{i \in \N_{\geq \sqrt{n}}} \frac{1}{i^2}
    \leq 2 e^2 n^{7/2} .
  \end{align*}

  Applying \Cref{thm:geometric-sum-bound} for $\lambda \coloneqq \frac{n^2}{40000} \ln n \leq \frac{\mu}{2}$, we get
  \begin{equation*}
    \Pr[T^{\star} \leq \lambda]
    \leq \Pr[T^{\star} \leq \mu - \lambda]
    \leq \exp\left(-\frac{\lambda^2}{2s}\right)
    \leq \exp\left(-\frac{1}{40000^2 \cdot 4e^2} \sqrt{n} \ln^2 n\right) .
  \end{equation*}
  By \cref{item:identical-stopping-times}, since $\lambda < n^3$, this bound carries over to~$T'$.

  Combining the converse probability of at least $1 - \exp\bigl(-\Theta(\sqrt{n} \log^2 n)\bigr)$ with the conditional probability of at least $1 - n^{-2/5}$ that we never make progress by flipping more than~$11$ bits within the first~$n^3$ iterations (for the appropriate range of~$\Xl$) as well as with the conditional probability of at least $1 - \exp\bigl(-\Theta(n^{3/5})\bigr)$ that $\Xl_{T_0} \geq \floor{n^{3/5}}$ and with the conditional probability of at least $1 - \exp\bigl(-\Theta(n)\bigr)$ from the Chernoff bound for the initialization, we conclude that the overall probability that $\Xl \geq \sqrt{n}$ for~$\lambda$ iterations is at least $1 - \Theta(n^{-2/5})$.
  This is in the desired order and thus concludes the case for~$\Xl$.
  Taking the union bound for the cases~$\Xl$ and~$\Xr$ concludes the case for the GSEMO and thus the proof.
\end{proof}

\idealizedProcessLowerBound*

\begin{proof}
  We proceed similar to the proof of \Cref{lem:time-for-expanding-toward-the-PF-borders} and consider the distance of the algorithm to the Pareto borders,~$d_{\mathrm{PF}}$, defined just before \Cref{lem:time-for-expanding-toward-the-PF-borders}.
  That is, we disregard whether the individuals we choose as parents are Pareto-optimal or not.
  We aim at bounding the number of iterations until this value is zero, that is, $S \coloneqq \inf \{t \in \N \mid d_{\mathrm{PF}}(P^{(t)}) = 0\}$.
  Note that~$S$ is stochastically dominated by the actual runtime of the algorithm, since a distance of zero to the Pareto borders is a necessary condition for the algorithm to cover the Pareto front.

  Formally, we consider $(X_t)_{t \in \N}$ such that for all $t \in \N$, we have $X_t \coloneqq d_{\mathrm{PF}}(P^{(t)})$.
  We aim at applying \Cref{thm:geometric-sum-bound} to the waiting times for~$X$ to decrease, which we define later.
  To this end, we only discuss global mutation, noting that the case for one-bit mutation works the same but simplifies many of the arguments (and leads to slightly better transition probabilities).

  Before we define the waiting times, we determine how many of them we need to consider with high probability.
  To this end, we first prove that with high probability, the algorithm does not decrease~$X$ by at least $\ceil{\sqrt{n} / 2}$ in a single iteration within the first~$n^3$ iterations.
  Note that this event implies that~$X$ does not skip the interval $[\floor{\sqrt{n} / 2} .. \ceil{\sqrt{n}}]$ within the first~$n^3$ iterations, recalling that we assume $X_0 \geq \ceil{\sqrt{n}}$.
  We bound the probability for iteration $t \in [0 .. n^3 - 1]$ that~$X_t$ decreases by at least~$\ceil{\sqrt{n} / 2}$.
  In order to decrease~$X_t$ by at least $\ceil{\sqrt{n} / 2}$, regardless of which individual we choose as parent, we
  have two possibilities: The first one is to flip some bits of the~$X_t$ bits of one value of which the parent has fewer.
  Since we need to decrease~$X_t$ by at least~$\ceil{\sqrt{n} / 2}$ and since there are at most~$X_t$ bits to choose from, this requires to flip $k \in [\floor{\sqrt{n} / 2} .. X_t]$ of these bits.
  Thus, recalling that by definition follows that $X_t \leq \floor{n / 2} / 2 \leq \frac{n}{4}$, the probability to do so is at most
  \begin{equation*}
    \binom{X_t}{k} \frac{1}{n^k}
    = \frac{1}{n^k} \cdot \prod\nolimits_{i \in [k]} \frac{X_t - i + 1}{i}
    \leq \left(\frac{X_t}{n}\right)^k
    \leq 4^{-k} ,
  \end{equation*}
  which is maximized for $k = \floor{\sqrt{n} / 2}$.

  The second possibility is to flip bits among the other $\floor{n / 2} - X_t$ bits.
  Flipping all $\floor{n / 2} - X_t$ of these bits results in $X_{t + 1} = 0$.
  And in order to make sure that $X_{t + 1} \leq X_t - \ceil{\sqrt{n} / 2}$, there must remain at most $X_t - \ceil{\sqrt{n} / 2}$ bits of this value after flipping, thus requiring to flip at least $\floor{n / 2} - X_t - (X_t - \ceil{\sqrt{n} / 2})$ bits.
  Thus, this case requires to flip $\ell \in [\floor{n / 2} - 2 X_t + \floor{\sqrt{n} / 2} .. \floor{n / 2} - X_t]$ bits.
  Hence, the probability for this case is at most
  \begin{equation*}
    \binom{\floor{n / 2} - X_t}{\ell} \frac{1}{n^\ell}
    = \frac{1}{n^\ell} \cdot \prod\nolimits_{i \in [\ell]} \frac{\floor{n / 2} - X_t - i + 1}{i}
    \leq \left(\frac{n/2}{n}\right)^\ell
    = 2^{- \ell} ,
  \end{equation*}
  which is maximized for $\ell = \floor{\sqrt{n} / 2}$.

  Combined with the other probability, this results in an overall probability of at most $2 \cdot 2^{-\floor{\sqrt{n} / 2}}$ for iteration~$t$.
  Via a union bound over all~$n^3$ iterations, the probability that~$X$ is never decreased by at least $\ceil{\sqrt{n} / 2}$ in a single iteration within the first~$n^3$ iterations is at least $1 - n^3 \cdot 2 \cdot 2^{-\floor{\sqrt{n} / 2}} \geq 1 - \exp\bigl(-\Theta(\sqrt{n})\bigr)$.
  In the following, we condition on this event.

  Similar to the calculations above, we show that with high probability, the (G)SEMO does not decrease~$X$ during the first~$n^3$ iterations by at least~$8$ in a single iteration once its value is at most~$\ceil{\sqrt{n}}$.
  Let $t \in [0 .. n^3 - 1]$ be the current iteration we consider.
  The probability to decrease~$X_t$ by at least~$8$ bits, regardless of which individual we choose as parent, requires to flip at least $k \in [8 .. X_t]$ out of the~$X_t$ bits that decrease~$X_t$, or to flip at least $\ell \in [\floor{n / 2} - 2 X_t + 8 .. \floor{n / 2} - X_t]$ bits of the other value.
  Thus, the probability to make progress of at least~$8$ during a single mutation is, similar to above, at most $2 \sqrt{n}^{-8} = 2 n^{-4}$.
  A union bound over all~$n^3$ iterations yields that~$X$ is never decreased by at least~$8$ in a single iteration during this time with probability at least $1 - n^3 \cdot 2 n^{-4} = 1 - 2n^{-1}$.
  In the following, we also condition on this event.

  We proceed with defining the waiting times we aim at bounding.
  We denote the random number of decreases of~$X$ (and thus the number of waiting times we aim to bound) by $J \in [\ceil{\sqrt{n}}]$.
  Let $T_0 \coloneqq \inf \{t \in \N \mid X_t \leq \ceil{\sqrt{n}}\}$, and for all $i \in [J - 1]$, let $T_i \coloneqq \inf \{t \in \N \mid t \geq T_{i - 1} \land X_t < X_{T_{i - 1}}\}$, and let $T_J \coloneqq \inf \{t \in \N \mid t \geq T_{J - 1} \land X_t = 0\}$.
  Furthermore, for all $i \in [J]$, let $D_i \coloneqq T_i - T_{i - 1}$.
  The $\{D_i\}_{i \in [J]}$ are our waiting times of interest, noting that $S = T_0 + \sum_{i \in [J]} D_i$.
  For convenience, we define $T^\star \coloneqq \sum_{i \in [J]} D_i$.

  In order to consider independent geometric random variables, we consider a \emph{variant} of the (G)SEMO that is identical to the original (G)SEMO as long as~$X$ is larger than~$\ceil{\sqrt{n}}$.
  Once~$X$ is at most~$\ceil{\sqrt{n}}$, the variant decreases~$X$ always by at most~$7$ in a single iteration, and whenever~$X$ decreases, the population of the variant contains individuals in the population with a distance of $i \in [X .. X + 6]$ to either Pareto border in the population; the variant takes the~$13$ individuals that were not created by mutation from other parts of the population that have a larger distance to the Pareto borders (noting that we have enough individuals to do so, as we assume that the population size is already linear in~$n$).
  The variant always picks the same individuals for mutation as the original algorithm, with the difference that if the original algorithm picks an individual that was swapped closer to the Pareto borders in the variant, the variant picks this updated individual instead.
  If the original algorithm picks an individual that is not present in the variant (which happens if the original algorithm decreases~$X$ by at least~$8$ in a single iteration), then the variant chooses a parent uniformly at random from its population.
  We use the same notation for the variant as we use for the original algorithm, but we use the indices ``$\mathrm{orig}$'' for the original algorithm and ``$\mathrm{var}$'' for the variant if it is not clear from context which algorithm we are referring to.

  For the variant introduced above, the variables $\{D_i\}_{i \in [J]}$ are independent, as each improvement of~$X$ results in the population containing the same individuals that can further improve~$X$, regardless of how the progress was achieved.
  Moreover, conditional on the original (G)SEMO never decreasing~$X$ by at least~$8$ during the first~$n^3$ iterations, since the variant has at least as many possibilities to decrease~$X$ as the original (G)SEMO in each iteration, we have for all $t \in [0 .. n^3 - 1]$ that $\Pr[S_{\mathrm{orig}} \leq t] \leq \Pr[S_{\mathrm{var}} \leq t]$.
  We continue by bounding $\Pr[S_{\mathrm{var}} \leq t]$ from above by \Cref{thm:geometric-sum-bound} for any feasible value of~$J$ by considering $\{D_i\}_{i \in [J]}$ for the (G)SEMO variant, noting that the latter follow each a geometric distribution, with success probabilities $\{p_i\}_{i \in [J]}$.
  For all $i \in [J]$, we bound~$p_i$ from above and below, recalling that we assume that the population size of the (G)SEMO and thus of the variant is already $\Theta(n)$.
  For the lower bound, we pick one specific individual with the smallest distance to a Pareto border as parent and flip exactly one bit such that~$X_{T_{i - 1}}$ decreases.
  Hence, we have $p_i \geq \frac{1}{\Theta(n)} (X_{T_{i - 1}} / n) (1 - \frac{1}{n})^6 \geq X_{T_{i - 1}} / \Theta(n^2)$.
  For the upper bound, we note that we always have~$14$ individuals ($7$ per Pareto border) that can be improved by flipping any appropriate number of bits.
  We assume that for each individual, flipping a single bit is sufficient, and we account for this probability~$7$ times per individual (instead of the appropriate numbers).
  Thus, we get $p_i \leq 7 \cdot \frac{14}{\Theta(n)} (X_{T_{i - 1}} / n) (1 - \frac{1}{n})^6 \leq X_{T_{i - 1}} / \Theta(n^2)$.
  Overall, we have $p_i = X_{T_{i - 1}} / \Theta(n^2)$.

  We continue by bounding $\mu \coloneqq \E[T^\star]$ from below.
  To this end, we note that since the variant only decreases~$X$ by at most~$7$ in each iteration, any realization of a run of the variant satisfies that $J \geq \floor{\ceil{\sqrt{n}} / 7}$.
  Thus, we bound by linearity of expectation and by \Cref{thm:sums-to-integrals},
  \begin{align*}
    \mu
    = \E\left[\sum\nolimits_{i \in [J]} D_i\right]
     & = \sum\nolimits_{i \in [J]} \E[D_i]
    \geq \Theta(n^2) \sum\nolimits_{i \in [\floor{\ceil{\sqrt{n}} / 7} - 1]} \frac{1}{7 i} \\
     & \geq \Theta(n^2) \int_{[1, \floor{\ceil{\sqrt{n}} / 7}]} \frac{1}{7i} \textrm{d} i
    \geq \Theta(n^2 \log n) .
  \end{align*}

  Moreover, we bound $s \coloneqq \sum_{i \in [J]} \frac{1}{p^2_i}$ from above as
  \begin{equation*}
    s
    = \sum\nolimits_{i \in [J]} \frac{1}{p^2_i}
    \leq \Theta(n^4) \sum\nolimits_{i \in [\ceil{\sqrt{n}}]} \frac{1}{i^2}
    \leq \Theta(n^4) \sum\nolimits_{i \in \N_{\geq \sqrt{n}}} \frac{1}{i^2}
    = \Theta(n^{7/2}).
  \end{equation*}

  Applying \Cref{thm:geometric-sum-bound} with $\lambda \coloneqq \Theta(n^2 \log n)$ such that $\lambda \leq \frac{\mu}{2}$, we get
  \begin{equation*}
    \Pr[T^\star \leq \lambda]
    \leq \Pr[T^\star \leq \mu - \lambda]
    \leq \exp\left(-\frac{\lambda^2}{2 s}\right)
    \leq \exp\bigl(-\Theta(\sqrt{n} \log^2 n)\bigr) .
  \end{equation*}

  Since $T^\star \leq S$, it follows that $\Pr[S \leq \lambda] \leq \Pr[T^\star \leq \lambda] \leq \exp\bigl(-\Theta(\sqrt{n} \log^2 n)\bigr)$.
  A union bound over this failure probability as well as the two failure probabilities of the events that we condition on yields that $\Pr[S > \lambda] \geq 1 - \exp\bigl(-\Theta(-\sqrt{n})\bigr) - 2 n^{-1} - \exp\bigl(-\Theta(\sqrt{n} \log^2 n)\bigr) \geq 1 - \Theta(n^{-1})$.
  This concludes the proof.
\end{proof}

\section{Runtime Analysis on OMM and OJZJ}
\label{sec:appendix:ommAndOjzj}

\subsubsection{\omm and \ojzj}

\timeToParetoFrontOMMOJZJ*

\begin{proof}
  The proof is very similar to that of \Cref{lem:time-to-Pareto-front}. 
  Denote by the random variable $T$ the time until there are at least $n/2$ individuals (which are also on the Pareto front). We show that $\Pr[T \geq 8en^2]=\exp(-\Omega(n))$.
  As long as $\vert{P^{(t)}}\vert < n/2$ there are two cases no matter which objective function is considered.
  \begin{itemize}
    \item[(1)] Suppose that there is an individual $x' \in P^{(t)}$ with $\ones{x'} \leq \floor{n/4}$ (i.e. $\ones{x'} = \floor{n/4}$ for $\ojzjk$). Since $\vert{P^{(t)}}\vert < n/2$, one finds $m \in [0,n/2]$ such that there is no individual $y' \in P^{(t)}$ with $\ones{y'}=n/4+m$. Hence, there is an individual $x \in P^{(t)}$ with $\ones{x} < n/2+\floor{n/4}$, but no $y$ with $\ones{y}=\ones{x}+1$.
    \item[(2)] Suppose that there is no individual $x' \in P^{(t)}$ with $\ones{x'} \leq \floor{n/4}$. Then one finds an individual $x \in P^{(t)}$ with $\ones{x}>n/4$, but no individual $y$ with $\ones{y}=\ones{x}-1$.
  \end{itemize}
  In both cases such a missing individual $y$ can be created by selecting $x$ as a parent and flipping one out of $\lceil{n/4}\rceil \geq n/4$ bits where the remaining bits are kept unchanged which happens (for both modified algorithms and problem classes) with probability at least $\frac{1}{(n+1)4e} \eqqcolon r$. Hence, $T$ is stochastically dominated by the sum of $\lceil{n/2}\rceil-1$ geometrically distributed random variables $V_1, \ldots , V_{\lceil{n/2}\rceil-1}$ with success probability $r$. Let $V\coloneqq \sum_{i=1}^{\lceil{n/2}\rceil-1}V_i$. Note that
  \[\mathbb{E}[V]=(\lceil{n/2}\rceil-1)(4en+4e) \leq n/2 \cdot (4en+4e) = 2en^2 + 2en \leq 4en^2.\]
  With \Cref{thm:geometric-sum-bound} we obtain for $\lambda\coloneqq 4en^2$, and $s\coloneqq (\lceil{n/2}\rceil-1)/r^2=16e^2(\lceil{n/2}\rceil-1)(n+1)^2$
  \begin{align*}
    \Pr[T \geq 8en^2] & \leq \Pr[V\geq 8en^2] \leq \Pr[V \geq \mathbb{E}[V]+4en^2]                                                       \\
                      & \leq \exp \left(-\frac{1}{4}\min\left\{\frac{\lambda^2}{s},\lambda r\right\}\right) = \exp(-\Omega(n)). \qedhere
  \end{align*}
\end{proof}

\timeForExpandingTowardThePFBordersOMMOJZJ*

\begin{proof}
  At first we investigate the modified (G)SEMO maximizing $\ojzjk$ for $\sqrt{n} \leq k < n/4$. We can assume that $n/4 \leq \ones{x^{(0)}} \leq 3n/4$ (by a Chernoff bound this happens with probability at least $1-\exp(-\Omega(n))$). Hence, the distance to the Pareto borders is at least $\sqrt{n}$ and one can decrease this distance only by flipping $\sqrt{n}$ specific bits at once which happens with probability at most $n^{-\sqrt{n}}$ (for both modified algorithms). By a union bound on $c n^2 \ln n$ generations for any constant $c$, we obtain the result. \\
  Now we investigate the modified (G)SEMO maximizing \omm or $\ojzjk$ for $1 < k < \sqrt{n}$.
  The proof is very similar to the proof of \Cref{lem:time-for-expanding-toward-the-PF-borders} with the minor difference that we consider $(X_t^{\leftarrow})_{t \in \mathbb{N}}$ with $X_t^{\leftarrow} = \min_{z \in P^{(t)}} \ones{z}$ and analogously $(X_t^{\rightarrow})_{t \in \mathbb{N}}$ with $X_t^{\rightarrow} = \min_{z \in P^{(t)}} n-\ones{z}$. Note that $d_{\text{PF}}(P^{(t)}) = \min\{X_t^{\rightarrow},X_t^{\leftarrow}\}$. By completely the same arguments from \Cref{lem:time-for-expanding-toward-the-PF-borders} for both modified algorithms, we obtain with probability $1-\Theta(n^{-2/5})$ that $X^{\leftarrow}_t$ is at least $\sqrt{n}$ for at least $cn^2 \ln n$ generations for a sufficiently small constant $c$.

  For $\ojzjk$, in case of $k > \sqrt{n}$, we note that the distance of the algorithm cannot attain, with the same probability bound as above, values outside of $[k .. n - k]$ because all individuals with a number of~$1$s outside of this interval (besides the all-$1$s and the all-$0$s bit string) are dominated by any solution inside this interval.
  Since the proof above shows that the algorithm only makes (with high probability) changes of at most~$\sqrt{n}$ to the number of~$1$s, the extremal solutions of the Pareto front cannot be reached during the considered time.
\end{proof}

\subsubsection{\omm}

\idealizedProcessLowerBoundOMM*

\begin{proof}
  As in the proof of \Cref{lem:idealized-process-lower-bound} we consider $S\coloneqq \inf\{t \in \mathbb{N} \mid d_{\text{PF}}(P^{(t)}) =0\}$ which is stochastically dominated by the actual runtime of the algorithm, since a distance of zero to the Pareto borders is a necessary condition for the algorithm to cover the whole Pareto front. By using the same arguments from \Cref{lem:idealized-process-lower-bound}, one can show that there is a sufficiently small constant $d$ such that $S>dn^2\log(n)$ with probability at least $1-\Theta(n^{-1})$.
\end{proof}

\gsemoLowerBoundOMM*

\begin{proof}
  By \Cref{lem:time-to-Pareto-front-OMM-OJZJ,lem:time-for-expanding-toward-the-PF-borders-OMM-OJZJ}, it follows with probability at least $1 - \Theta(n^{-1})$ that the (G)SEMO reaches the Pareto front with a population size of at least~$n / 2$ while the distance to the Pareto borders is still at least~$\sqrt{n}$.
  By \Cref{lem:idealized-process-lower-bound-OMM}, it takes the (G)SEMO $\Omega(n^2 \log n)$ objective-function evaluations until it covers the entire Pareto front, concluding the proof.
\end{proof}

\subsubsection{\ojzj}

\idealizedProcessLowerBoundOJZJ*

\begin{proof}
  In order to create $1^n$, one has to choose an individual $x$ with $\ones{x}=n-k$ and flip $k$ specific bits or another individual $y$ with $\ones{y}<n-k$ and flip $\ones{y} \geq k+1$ specific bits. This happens with probability at most $1/\Theta(n) \cdot 1/n^k + 1/n^{k+1} = O(1/n^{k+1})$ which gives $\Omega(n^{k+1})$ objective evaluations in expectation.
\end{proof}

\gsemoLowerBoundOJZJ*

\begin{proof}
  By \Cref{thm:chernoff-lower-bound}, with probability at least $1 - \exp(-\Theta(n))$, the initial solution of the (G)SEMO has at least~$k$ ones and at most $n - k$ ones.
  Then, as in the proof of \Cref{thm:gsemo-lower-bound-OMM}, the (G)SEMO reaches with probability at least $1 - \Theta(n^{-1})$ the Pareto front with a population size of at least~$n / 2$.
  The lower bound follows then by \Cref{lem:idealized-process-lower-bound-OJZJ}.
\end{proof}

\end{document}